\documentclass[conference]{IEEEtran}
\usepackage{cite}
\usepackage{amsmath,amssymb,amsfonts}
\usepackage{algorithmic}
\usepackage{graphicx}
\usepackage{textcomp}
\usepackage{xcolor}
\usepackage{pifont}
\usepackage{}
\usepackage{bbding}
\usepackage{amssymb}
\usepackage{amsfonts}
\usepackage{cite}
\usepackage{cases}
\usepackage{txfonts}
\usepackage{subfigure}
\usepackage[ruled]{algorithm2e}
\usepackage{thmbox}
\usepackage{hyperref}
\usepackage{multicol}
\usepackage{threeparttable}
\usepackage[sort,numbers]{natbib}
\usepackage{graphicx}
\usepackage{subfigure} 

\usepackage[T1]{fontenc}
\usepackage{graphicx}
\usepackage{array}
\newcommand{\PreserveBackslash}[1]{\let\temp=\\#1\let\\=\temp}
\newcolumntype{C}[1]{>{\PreserveBackslash\centering}p{#1}}
\newcolumntype{R}[1]{>{\PreserveBackslash\raggedleft}p{#1}}
\newcolumntype{L}[1]{>{\PreserveBackslash\raggedright}p{#1}}

\newtheorem{lemma}{Lemma}
\newtheorem{theorem}{Theorem}
\newtheorem{proof}{Proof}
\IEEEoverridecommandlockouts

\begin{document}
\title{Enhancing QoS in Edge Computing through Federated Layering Techniques: A Pathway to Resilient AI Lifelong Learning Systems
}

\author{\IEEEauthorblockN{
Chengzhuo Han\IEEEauthorrefmark{1}\IEEEauthorrefmark{2}, 
} \\
\IEEEauthorblockA{\IEEEauthorrefmark{1}  School of Cyber Science and Engineering, Southeast University, China 
\IEEEauthorrefmark{2} Peng Cheng Laboratory, China \\
\IEEEauthorrefmark{
3}  Navigation College, Dalian Maritime University, China \\
\\
\{hcz\_dmu@163.com\}}

}


\maketitle
\begin{abstract}

In the context of the rapidly evolving information technology landscape, marked by the advent of 6G communication networks, we face an increased data volume and complexity in network environments. This paper addresses these challenges by focusing on Quality of Service (QoS) in  edge computing frameworks. We propose a novel approach to enhance QoS through the development of General Artificial Intelligence Lifelong Learning Systems, with a special emphasis on Federated Layering Techniques (FLT). Our work introduces a federated layering-based small model collaborative mechanism aimed at improving AI models' operational efficiency and response time in environments where resources are limited. This innovative method leverages the strengths of cloud and edge computing, incorporating a negotiation and debate mechanism among small AI models to enhance reasoning and decision-making processes. By integrating model layering techniques with privacy protection measures, our approach ensures the secure transmission of model parameters while maintaining high efficiency in learning and reasoning capabilities. The experimental results demonstrate that our strategy not only enhances learning efficiency and reasoning accuracy but also effectively protects the privacy of edge nodes. This presents a viable solution for achieving resilient large model lifelong learning systems, with a significant improvement in QoS for edge computing environments.

\end{abstract}


\IEEEpeerreviewmaketitle

\section{Introduction}\label{sec:introduction}

In the wake of rapid advancements in technology and the emergence of 6G communication networks \cite{AI6GTaxonomy2023,nguyen2021federated}, we are transitioning into an era marked by an exponential increase in data and networking capabilities. This new age not only demands more sophisticated information processing, data transmission, and intelligent applications but also presents novel challenges. Amidst this evolution, the role of AI models in edge computing is becoming increasingly critical, particularly in their ability to process complex data structures and adapt to a wide array of application scenarios, from smart infrastructure to real-time analytics \cite{chen2023bigAI6G}.

However, the path is fraught with obstacles, including volatile learning environments, limited resources, security vulnerabilities, and the diverse nature of edge devices, all of which can undermine the efficiency and reliability of AI systems. To navigate these challenges, this paper introduces a federated layering approach to facilitate small model collaboration, specifically designed to augment the operational efficiency and responsiveness of AI models within the constrained confines of edge computing \cite{AIGC6GEra2023}. Leveraging a cloud-edge collaborative framework, this strategy harnesses the complementary strengths of cloud and edge computing to not only boost the performance of AI models but also refine the decision-making processes through a dynamic interplay of negotiation and debate among smaller, distributed AI entities \cite{AIPowered6G2023}.

Central to our investigation is the development of resilient systems capable of maintaining consistent performance amidst the fluctuating dynamics of network environments and emerging security threats \cite{AI6GVision2023,rauniyar2023federated}. By embracing cloud-edge synergy, our architecture achieves greater efficiency and adaptability in processing vast datasets. The incorporation of a negotiation and debate mechanism among small AI models introduces an innovative layer, fostering a collaborative environment where models can exchange insights and experiences while preserving their unique strengths, thereby amplifying the system's collective intelligence, adaptability, and resilience \cite{6GIoTAdvances2023}.

Given the paramount importance of data privacy and security in this context, our discourse extends to exploring robust data protection strategies for AI models within the edge computing paradigm \cite{EmergingTech6G2023}. We propose a secure parameter transmission mechanism that integrates Federated Layering Techniques (FLT) with advanced privacy measures, effectively ensuring the integrity and confidentiality of data during transmission. This approach not only fortifies data security but also underpins the resilience of the entire system \cite{AIApplications6G2023}.

Through empirical analysis, we demonstrate significant advancements in learning efficiency, inference accuracy, and the safeguarding of privacy and security at the edge, thereby offering a pragmatic blueprint for constructing resilient, large-scale lifelong learning systems. These insights pave the way for future innovations in technology, particularly in bolstering system resilience and security amidst evolving environmental conditions.

The ensuing sections will delve into our research methodology, empirical findings, and overarching conclusions. Our exploration not only addresses contemporary technological hurdles but also charts new trajectories for the advancement of communication technologies and AI, with a particular focus on enhancing QoS in edge computing. These contributions, both theoretical and practical, are poised to influence the broader trajectory of information technology development, especially as 6G networks continue to proliferate and integrate into various sectors, catalyzing societal and technological transformations.

The key contributions of this paper are encapsulated as follows:
\begin{itemize}
    \item \textbf{Federated Layering for Small Model Collaboration:} We unveil a federated layering technique that enables small AI models to collaborate efficiently within edge computing environments, optimizing AI operations and decision-making through a unique inter-model negotiation and debate framework, thereby enhancing computational efficiency and adaptability to complex scenarios.
    \item \textbf{Synergistic Cloud-Edge Architecture:} Our proposed architecture amalgamates the computational prowess of cloud computing with the real-time capabilities of edge computing, significantly improving operational efficiency and flexibility for AI models, especially under scenarios of high data volume and critical latency requirements, making it indispensable for time-critical applications.
    \item \textbf{Enhanced Privacy and Secure Transmission:} In response to the escalating concerns around data privacy, our approach integrates federated layering with stringent privacy protections, ensuring secure parameter transmission during the AI learning and reasoning processes, thereby safeguarding data privacy and system security, particularly in contexts involving sensitive information.
\end{itemize}

\section{System Model}\label{sec:system model}
The system model presented herein is formulated to mitigate the intricacies arising from the rapid evolution of information technology within the burgeoning era of 6G communication networks. The model is intricately composed, incorporating avant-garde methodologies in distributed deployment, collaborative mechanisms, and privacy-preserving security, thus contributing significantly to the formulation of resilient and efficient communication large models. Let us delve into the mathematical formulations underpinning the key components of this sophisticated system model:

\subsection{Distributed Deployment Architecture}
The Distributed Deployment Architecture serves as the foundation of our proposed system model, illustrated in Fig. \ref{mode}. This architecture is carefully crafted to enhance computational efficiency and minimize latency, achieved through strategic task distribution between cloud and edge computing environments. The following discussion delves into the detailed mathematical framework that underlies this architecture, shedding light on its sophisticated design principles.

\subsubsection{Objective Function Formulation}
The overarching objective is to minimize the combined cost of computation at the cloud \( C \) and edge \( E \) while considering the latency associated with data transmission. This objective is mathematically formulated as:
\begin{equation} \text{minimize}_{C,E} \text{Cost}(C) + \text{Cost}(E) + \text{Latency}(C, E) \end{equation}
Here, \( \text{Cost}(C) \) and \( \text{Cost}(E) \) denote the computational costs at the cloud and edge, respectively, and \( \text{Latency}(C, E) \) represents the latency incurred during data transmission between the cloud and edge.

\subsubsection{Computational Cost Formulation}
The computational cost at the cloud and edge involves the processing of tasks denoted as \( T \). Let \( \text{Comp}(C, T) \) and \( \text{Comp}(E, T) \) represent the computational costs at the cloud and edge for task set \( T \). The total computational cost is given by:
\begin{equation} \text{Cost}(C) = \sum_{T} \text{Comp}(C, T) \text{ and } \text{Cost}(E) = \sum_{T} \text{Comp}(E, T) \end{equation}

\subsubsection{Latency Formulation}
The latency during data transmission between cloud and edge is influenced by the amount of data transmitted denoted as \( D \). Considering the transmission rate \( R \) and the distance \( D \) between cloud and edge, the latency is defined as:
\begin{equation} \text{Latency}(C, E) = \sum_{D} \frac{D}{R} + \frac{D}{R} \end{equation}

\subsubsection{Optimization Problem}
Combining the computational cost and latency formulations, the optimization problem becomes:
\begin{equation} \text{minimize}_{C,E} \sum_{T} (\text{Comp}(C, T) + \text{Comp}(E, T)) + \sum_{D} \frac{2D}{R} \end{equation}
Subject to resource constraints and data transmission limitations between cloud and edge.

\subsubsection{Resource Constraints}
The computational resources at the cloud and edge are constrained, represented as \( R_C \) and \( R_E \) respectively. The constraints are formulated as:
\begin{equation}\text{Comp}(C, T) \leq R_C \text{ and } \text{Comp}(E, T) \leq R_E \end{equation}

\subsubsection{Data Transmission Limitations}
The amount of data transmitted between cloud and edge is constrained by the available bandwidth denoted as \( B \). The constraint is expressed as:
\begin{equation} \sum_{D} D \leq B \end{equation}

This detailed formulation of the Distributed Deployment Architecture encapsulates the intricate mathematical relationships governing the allocation of tasks, computational costs, and latency considerations in the optimization process, contributing to the overall resilience and efficiency of the system model.

\begin{figure}
  \centering
  \includegraphics[scale = 0.6]{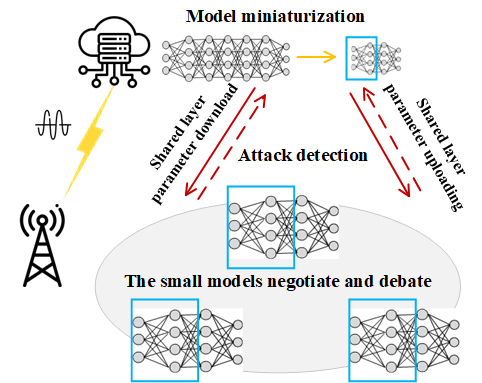}
  \caption{6G communication networks architecture based on FLT
}\label{mode}
\end{figure}

\subsection{Collaborative Mechanism among Small AI Models}
The Collaborative Mechanism is a pivotal component of the proposed system, fostering synergy among small AI models to enhance operational efficiency. Let's embark on a mathematical exploration of the intricacies of this collaborative process:

\subsubsection{Knowledge Sharing}
The core idea is to facilitate knowledge sharing among small AI models denoted as \( M \). Let \( \text{K}(M_i, M_j) \) represent the knowledge shared from model \( M_j \) to model \( M_i \). The collaborative mechanism encourages models to exchange knowledge, leading to an enriched collective learning experience:
\begin{equation} \text{K}(M_i, M_j) = \alpha \cdot \text{Experience}(M_j) + \beta \cdot \text{Expertise}(M_j) \end{equation}
Here, \( \alpha \) and \( \beta \) are weights assigned to the experience and expertise components, respectively.

\subsubsection{Learning Enhancement}
The collaborative mechanism aims to enhance the learning capabilities of small AI models through collaborative learning. The learning enhancement factor \( \Lambda(M_i) \) for model \( M_i \) is formulated as a function of the accumulated knowledge from other models:
\begin{equation} \Lambda(M_i) = \sum_{M_j} \text{K}(M_i, M_j) \end{equation}
This collaborative learning process contributes to the overall improvement of individual models within the collaborative ecosystem.

\subsubsection{Decision Optimization}
In decision-making scenarios, the collaborative mechanism involves models engaging in negotiation and debate. Let \( \text{Decision}(M_i) \) represent the decision made by model \( M_i \). The collaborative decision is influenced by the consensus reached during negotiation:
\begin{equation} \text{Decision}(M_i) = \frac{1}{N} \sum_{M_j} \text{Negotiate}(M_i, M_j) \end{equation}
Here, \( \text{Negotiate}(M_i, M_j) \) denotes the negotiation outcome between models \( M_i \) and \( M_j \), and \( N \) is the total number of models involved in the collaboration.

\subsubsection{Overall Collaboration Objective}
The overarching objective of the collaborative mechanism is to optimize the overall performance of the system. The collaborative objective function is defined as:
\begin{equation} \max \left\{ \sum_{M_i} \Lambda(M_i) + \sum_{M_i} \text{Decision}(M_i) \right\} \end{equation}
Subject to constraints that ensure fair collaboration, prevent information imbalance, and address resource limitations for individual models.

This mathematical elucidation provides a comprehensive understanding of the Collaborative Mechanism among Small AI Models, shedding light on the intricacies of knowledge sharing, learning enhancement, and decision optimization within the collaborative ecosystem.

\subsection{Privacy-Preserving Parameter Security Mechanism}

The PPPSM ensures the confidentiality of model parameters during transmission. Below we present the simplified mathematical foundations of this security measure.

\subsubsection{Secure Transmission Protocol}
The mechanism utilizes a cryptographic protocol for secure transmission:
\begin{equation} 
\text{SecTrans}(M_i, M_j) = \text{Enc}(\text{Params}(M_i), \text{Key}(M_j)) 
\end{equation}
Here, \(\text{Enc}\) denotes the encryption process using the recipient model's key.

\subsubsection{Privacy-Preserving Objective}
The goal is to enhance privacy-preserving security:
\begin{equation} 
\max \sum_{M_i} \text{Sec}(M_i) 
\end{equation}
This is subject to constraints for transmission efficiency and model accuracy.

\subsubsection{Security Metric}
The security metric, \(\text{Sec}\), combines encryption strength and transmission efficiency:
\begin{equation} 
\text{Sec}(M_i) = \gamma \cdot \text{EncStr}(M_i) + (1 - \gamma) \cdot \text{TransEff}(M_i) 
\end{equation}
The parameter \(\gamma\) represents the balance between encryption strength and efficiency.

\subsubsection{Homomorphic Encryption}
The mechanism adopts homomorphic encryption for computations on encrypted data:
\begin{equation} \begin{array}{l}
{\rm{HomEnc}}({M_i}) = {\rm{Dec}}({\rm{Comp}}({\rm{Enc}}({\rm{Params}}({M_i}),\\
{\rm{Key}}({M_j})),{\rm{Key}}({M_i})))
\end{array}
\end{equation}
Homomorphic encryption (\(\text{HomEnc}\)) maintains confidentiality during computations.

This section highlights the essential aspects of the PPPSM, focusing on secure transmission, encryption balance, and homomorphic encryption in ensuring data security.

\subsection{Federated Layering Techniques for Lifelong Learning Systems}
\subsubsection{Layered Model Representation}
Each model \(M_i\) in the federated system is represented as a layered structure:
\begin{equation} M_i = \{L_{i1}, L_{i2}, ..., L_{in}\} \end{equation}
Here, \(L_{ij}\) denotes the \(j\)-th layer in the \(i\)-th model.

\subsubsection{Federated Model Aggregation}
The aggregation of models across layers is expressed as:
\begin{equation} \text{Aggr}(L_j) = \frac{1}{N} \sum_{i=1}^{N} L_{ij}\end{equation}
This ensures a collaborative learning process where layers from different models contribute to the overall system knowledge.

\subsubsection{Resilience Objective}
The objective is to maximize the resilience of the federated system:
\begin{equation} \max \left\{ \sum_{j=1}^{n} \text{Resi}(\text{Aggr}(L_j)) \right\} \end{equation}
Subject to constraints ensuring balanced learning across models and layers.

\subsubsection{Resilience Metric}
The resilience metric (Resilience) is defined as a function of layer-wise learning stability:
\begin{equation} \text{Resi}(L_j) = \beta \cdot \text{Lear}(L_j) + (1 - \beta) \cdot \text{Adap}(L_j) \end{equation}
Here, \(\beta\) modulates the trade-off between learning stability and adaptability.

\subsubsection{Layer-wise Adaptability}
Adaptability of each layer is expressed as:
\begin{equation} \text{Adap}(L_j) = \text{UpdateRate}(L_j) \times \text{Comp}(L_j)\end{equation}
Layer-wise adaptability ensures that the system dynamically adjusts to new information while maintaining compatibility.

This chapter unveils the FLT, shedding light on the layered model representation, federated aggregation, and the nuanced interplay of resilience metrics. The expressive power of these techniques lies in their ability to foster collaborative and resilient learning across diverse models within Lifelong Learning Systems.

\subsection{Scheme Assumptions}
	According to the scenario and security objective proposed by this scheme, we firstly assume that the training set and test set are identically distributed \cite{han2022out}, and that the data distributions of different data parties are different but have some similarity. This allows each party to train its own personalized model and ensures that different data parties can learn from each other.
Secondly, we assume that the edge nodes are able to collect a portion of data with strong security and make a dection data set with them. The data can be derived from the edge nodes or from the terminal devices that have been authenticated or security verified. 
These data do not need to be large as they are only used to detect anomalies in the uploaded model parameters.
This reduces the contribution of untrustworthy parameters and ensures the model's accuracy. With an aim to ensure the model's convergence, we finally suppose that the loss function for training the local model on each terminal device is
convex function. In this case, the optimal solution of the model is acquired through the gradient descent algorithm.

\section{Algorithm Design}\label{sec:algorithm}
The algorithm design for the proposed FLT for Lifelong Learning Systems encapsulates the intricate process of collaborative and resilient learning. Let's delve into the details of the algorithmic framework.
\subsection{Algorithm Initialization}
The initialization phase of the Federated Layering Techniques for Lifelong Learning Systems (FLT) algorithm involves setting up the models and parameters.

\subsubsection{Model Initialization}
\begin{itemize}
    \item \textbf{Initialize Layered Models:} Define \( N \) models, each with a layered structure denoted as \( M_i = [L_{i1}, L_{i2}, ..., L_{im}] \), where \( L_{ij} \) represents the \( j \)-th layer of model \( M_i \).
    \item \( M_i = [L_{i1}, L_{i2}, ..., L_{im}], \forall i \in [1, N] \)
\end{itemize}

\subsubsection*{Parameter Initialization}
\begin{itemize}
    \item \textbf{Initialize Learning Rates:} Set the initial learning rates for each layer of each model. Let \( LR_{ij} \) represent the learning rate for layer \( j \) of model \( i \).
    \item \( LR_{ij} \sim \text{Uniform}(LR_{\text{min}}, LR_{\text{max}}), \forall i, j \)
    \item \textbf{Initialize Adaptability Coefficients:} Assign initial adaptability coefficients to control the adaptability of each layer. Let \( \text{Adapt}_{ij} \) represent the adaptability coefficient for layer \( j \) of model \( i \).
    \item \( \text{Adapt}_{ij} \sim \text{Uniform}(\text{Adapt}_{\text{min}}, \text{Adapt}_{\text{max}}), \forall i, j \)
    \item \textbf{Initialize Compatibility Thresholds:} Define compatibility thresholds to regulate collaboration among layers. Let \( \text{Comp}_{ij} \) denote the compatibility threshold for layer \( j \) of model \( i \).
    \item \( \text{Comp}_{ij} \sim \text{Uniform}(\text{Comp}_{\text{min}}, \text{Comp}_{\text{max}}), \forall i, j \)
\end{itemize}

The model and parameter initialization sets the foundation for the subsequent collaborative learning process.

\subsection{Algorithm Process}
The FLT algorithm consists of iterative training and collaboration phases in Fig. \ref{fig:liuchengtu}.
\begin{figure*}
  \centering
  \includegraphics[scale = 0.5]{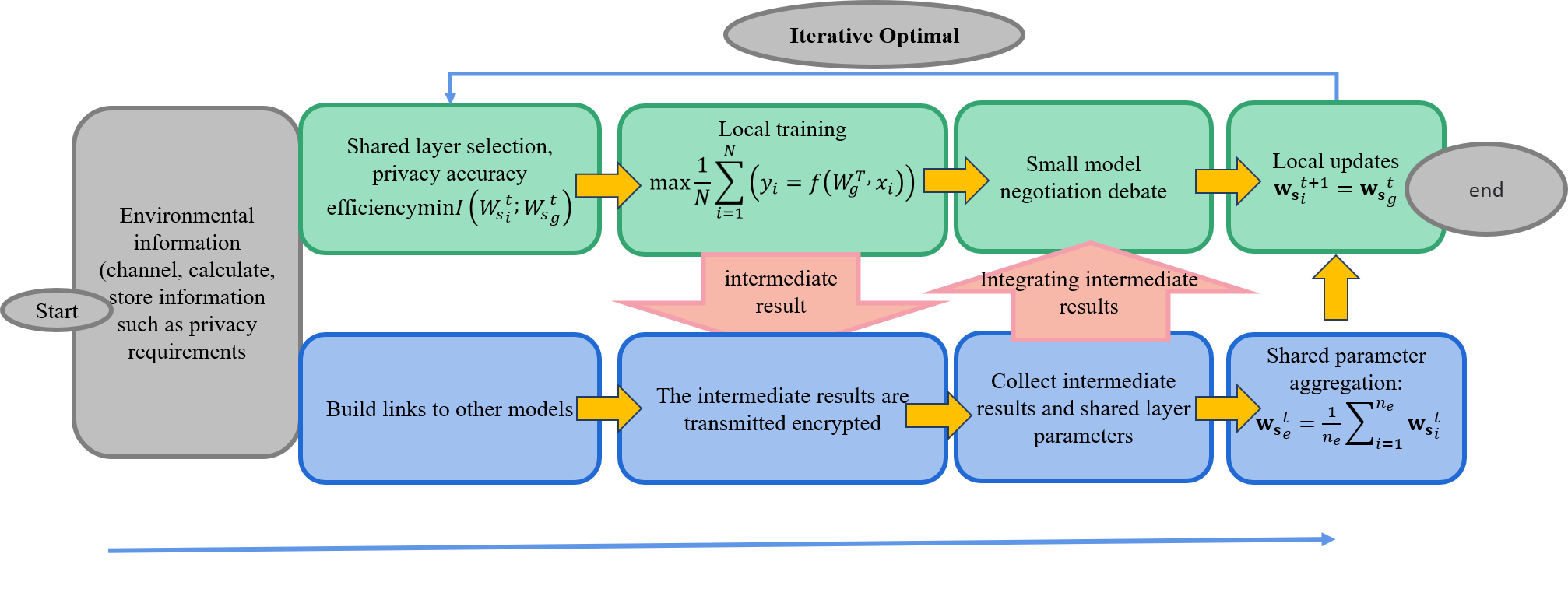}
  \caption{6G communication networks architecture based on FLT
}\label{fig:liuchengtu}
\end{figure*}
\subsubsection{Training Iterations as Algorithm \ref{alg:Initialization}}
\begin{itemize}
    \item \textbf{Forward Pass:} For each model \( M_i \), perform a forward pass to obtain predictions \( \hat{Y}_i \).
    \item \( \hat{Y}_i = \text{Forward}(M_i, X), \forall i \)
    \item \textbf{Calculate Loss:} Compute the loss \( L_i \) comparing predictions \( \hat{Y}_i \) with ground truth \( Y \).
    \item \( L_i = \text{Loss}(\hat{Y}_i, Y), \forall i \)
    \item \textbf{Backward Pass:} Conduct a backward pass to calculate gradients and update model parameters.
    \item \( \nabla L_i = \text{Backward}(M_i, \nabla L_i), \forall i \)
    \item \textbf{Update Parameters:} Adjust model parameters using the calculated gradients and learning rates.
    \item \( M_i = \text{Update}(M_i, \nabla L_i, LR_{ij}), \forall i, j \)
\end{itemize}

\subsubsection{Collaboration Phase}
The comprehensive algorithm for detecting node anomalies is detailed in Algorithm \ref{alg:FLTTrain}.
\begin{itemize}
    \item \textbf{Layer Collaboration:} Evaluate the compatibility of layers and initiate collaboration.
    \item \( \text{Collaborate}(M_i, M_j, \text{Comp}_{ij}), \forall i, j \) s.t. \( i \neq j \)
    \item \textbf{Negotiation and Debate:} Engage in negotiation and debate mechanisms to share knowledge.
    \item \( \text{NegotiateDebate}(M_i, M_j), \forall i, j \) s.t. \( i \neq j \)
\end{itemize}

The collaborative learning process enhances the adaptability and robustness of the models, continuing until convergence.
\subsection{Convergence Analysis}
To establish the convergence of the FLT for Lifelong Learning Systems algorithm, we employ mathematical formulations and principles from optimization theory. The objective is to demonstrate that the algorithm converges to a stable solution over training iterations.

\subsubsection{Notation}
\begin{itemize}
    \item \( M_i \): Small AI model \( i \).
    \item \( L_i \): Loss function for model \( M_i \).
    \item \( \theta_i \): Parameters of model \( M_i \).
    \item \( \nabla L_i \): Gradient of the loss function \( L_i \) with respect to \( \theta_i \).
    \item \( LR_{ij} \): Learning rate for updating model \( M_i \) using model \( M_j \).
    \item \( \text{Comp}_{ij} \): Compatibility measure between layers of models \( M_i \) and \( M_j \).
\end{itemize}

\subsubsection{Convergence Theorem}
\textbf{Theorem:} The FLT algorithm converges to a stationary point.

\textbf{Proof Outline:}
\begin{enumerate}
    \item \textbf{Objective Function Bound:} Demonstrate that the loss function \( L_i \) is bounded for all models \( M_i \) during training.
    \begin{equation} 0 \leq L_i \leq B, \forall i \end{equation}
    \item \textbf{Gradient Bound:} Show that the gradients \( \nabla L_i \) are bounded.
    \begin{equation} \| \nabla L_i \| \leq G, \forall i \end{equation}
    \item \textbf{Learning Rate Conditions:} Confirm that learning rates \( LR_{ij} \) satisfy certain conditions.
    \begin{equation} \sum_{i,j} LR_{ij} < \infty, \sum_{i,j} (LR_{ij})^2 < \infty \end{equation}
    \item \textbf{Layer Compatibility Conditions:} Establish conditions for layer compatibility \( \text{Comp}_{ij} \) ensuring collaboration does not destabilize the training.
    \begin{equation} \text{Comp}_{ij} \leq \text{Threshold}, \forall i, j \text{ s.t. } i \neq j \end{equation}
    \item \textbf{Convergence Criteria:} Utilize the Lyapunov function to demonstrate that the objective function decreases over time.
    \begin{equation} V(t) = \sum_i L_i(t) \text{ with } V'(t) \leq -\epsilon V(t) \end{equation}
\end{enumerate}

The convergence theorem establishes that the algorithm reaches a stationary point where the loss functions stabilize. The proof involves bounding the loss and gradients, ensuring suitable learning rates, and imposing compatibility conditions for effective collaboration. Together, these conditions guarantee the convergence of the FLT algorithm, providing a solid theoretical foundation for its effectiveness in training resilient lifelong learning systems.

\begin{algorithm}[]
    \caption{Initialization and Training of FLT Algorithm}
\label{alg:Initialization}
    \KwIn{learning rate $\alpha$, number of clients $N$, number of epochs $E$}
    \KwOut{Trained global model $global\_model$, Losses $losses$}

    Initialize $learning\_rate \gets 0.01$\;
    Initialize $num\_clients \gets 10$\;
    Initialize $epochs \gets 100$\;

    Initialize global model $global\_model \gets Initialize\_Global\_Model{}$\;
    Initialize local models array $local\_models \gets []$\;
    \For{$i \gets 1$ \textbf{to} $num\_clients$}{
        $local\_models[i] \gets Initialize\_Global\_Model{}$\;
    }
    Initialize losses list $losses \gets []$\;

    \For{$epoch \gets 1$ \textbf{to} $epochs$}{
        Update global model $global\_model \gets FederatedAverage{local\_models}$\;
        \For{$n \gets 1$ \textbf{to} $num\_clients$}{
            Update local model and calculate loss $(local\_model[n], loss)\gets Train\_Local\_Model\{local\_models[n], global\_model, \alpha\}$\;
            Append $loss$ to $losses$\;
        }
    }
\end{algorithm}
\begin{algorithm}[]
    \caption{Federated Layering Techniques Training Algorithm}
    \label{alg:FLTTrain}
    \LinesNumbered
    \KwIn{global model $global\_model$, local models $local\_models$, learning rate $\alpha$, number of epochs $E$}
    \KwOut{Trained global model $global\_model$, Losses $losses$}

    Initialize $losses \gets []$\;
    \For{$epoch \gets 1$ \textbf{to} $E$}{
        Update global model $global\_model \gets FederatedAverage(local\_models)$\;
        \For{$client\_model$ \textbf{in} $local\_models$}{
            Update local model and calculate loss $(client\_model, loss) \gets TrainLocalModel(client\_model, global\_model, \alpha)$\;
            Append $loss$ to $losses$\;
        }
    }
    \Return{$global\_model$, $losses$}\;

    \SetKwProg{Fn}{Function}{:}{end}
    \Fn{FederatedAverage}{$models$}{
        Compute $global\_model\_params \gets AverageParams(models)$\;
        Set $global\_model$ parameters using $global\_model\_params$\;
        \Return{$global\_model$}\;
    }
    \textbf{End Function}

    \Fn{TrainLocalModel}{$local\_model$, $global\_model$, $\alpha$}{
        Set $local\_model$ parameters using $global\_model$ parameters\;
        Compute $loss \gets FederatedLayeringTraining(local\_model, \alpha)$\;
        \Return{$local\_model$, $loss$}\;
    }
    \textbf{End Function}

    \Fn{FederatedLayeringTraining}{$model$, $\alpha$}{
        Compute $layer\_outputs \gets model.FederatedForward()$\;
        Compute $layer\_losses \gets ComputeLayerLosses(layer\_outputs)$\;
        Update $model$ using $layer\_losses$ and $\alpha$\;
        Update $model$ global parameters\;
        Compute $total\_loss \gets Sum(layer\_losses)$\;
        \Return{$total\_loss$}\;
    }
    \textbf{End Function}
\end{algorithm}
\subsection{Implementation of the FLT Algorithm for Anomaly Detection}

The process of anomaly detection in our federated system is depicted in Fig. \ref{workflow}. This figure outlines the sequence of steps in the detection scheme.

\begin{figure}[htbp]
    \centering
    \includegraphics[scale = 0.25]{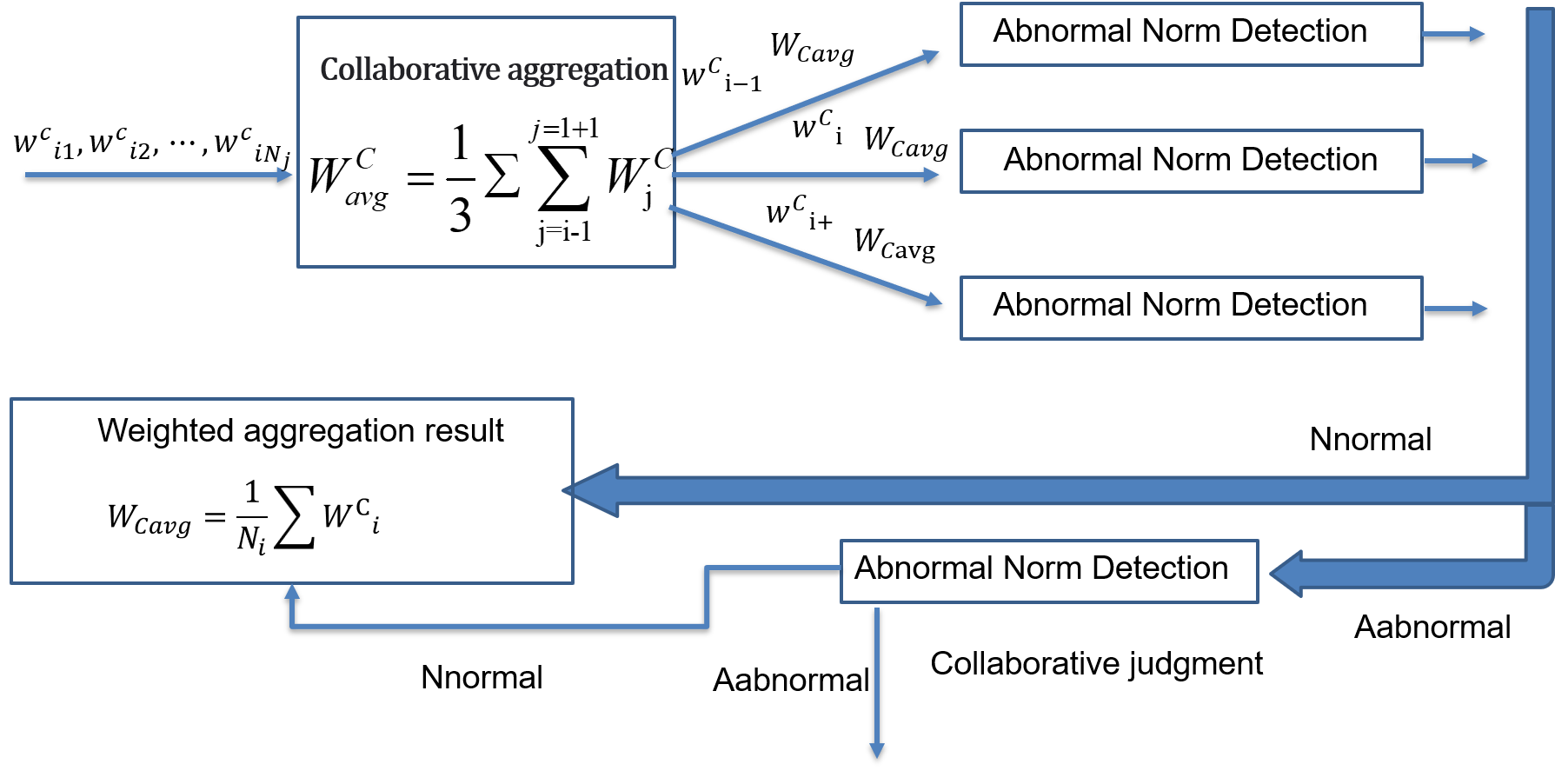}
    \caption{The Workflow of Anomaly Detection Process}\label{workflow}
\end{figure}

Each node, labeled as \(Node_x\), collects aggregated layer weights \(p_{xj}^a\) from device \(D_{xj}\). The process begins with averaging these weights to calculate the average weight \(p_{mean}^a\) as shown in Equation \ref{eq_bench}:

\begin{equation}
    p_{mean}^a = \frac{1}{{M_x}}\sum_{j = 1}^{M_x} p_{xj}^a  \label{mean_agg}
\end{equation}

Next, we create a benchmark \(p_{xj}^{a'}\) for each parameter \(p_{xj}^a\), using the mean value \(p_{mean}^b\) for comparison:

\begin{equation}
    p_{xj}^{a'} = \frac{1}{{M_x - 1}}\left( \sum_{k = 1}^{M_x} p_{xk}^a - p_{xj}^a \right) = \left( p_{mean}^a - \frac{p_{xj}^a}{{M_x}} \right) * \frac{{M_x}}{{M_x - 1}} \label{eq_bench}
\end{equation}

The variance distance \(v_{xj}^a\) between the input parameter \(p_{xj}^a\) and the benchmark \(p_{xj}^{a'}\) is determined:

\begin{equation}
    v_{xj}^a = \| p_{xj}^a - p_{xj}^{a'} \|
\end{equation}

The mean of \(v_{xj}^a\) across all nodes is then calculated:

\begin{equation}
    v_{x_{mean}}^a = \frac{1}{{M_x}}\sum_{j = 1}^{M_x} v_{xj}^a
\end{equation}

Similarly, we compute \(v_{xj}^{a'}\) as the aggregated result excluding \(v_{xj}^a\):

\begin{equation}
    v_{xj}^{a'} = \frac{1}{{M_x - 1}}\left( \sum_{k = 1}^{M_x} v_{xk}^a - v_{xj}^a \right) = \left( v_{x_{mean}}^a - \frac{v_{xj}^a}{{M_x}} \right) * \frac{{M_x}}{{M_x - 1}}
\end{equation}

Anomalies are flagged based on the deviation of \(v_{xj}^a\) from a predefined threshold \(\theta\). If \(v_{xj}^a\) surpasses \(\theta v_{x_{mean}}^a\), the parameter is marked as suspicious.

To evaluate accuracy, private layer weights \(q_{xj}^p\) are uploaded by \(D_{xj}\). The weights \(p_{xj}^{a'}\) and \(p_{xj}^a\) are combined with \(q_{xj}^p\) to form complete model parameters \(q_{xj}^{e'}\) and \(q_{xj}^{e}\), respectively. Training on a local dataset \(data_x^{test}\), we calculate the prediction accuracies for various categories. The difference in accuracy is quantified by \(\frac{{\phi _{xj}^k - \phi _{xj}^{k'}}}{{\phi _{xj}^{k'}}}\), where \(\phi _{xj}^k\) and \(\phi _{xj}^{k'}\) denote the category-specific accuracies.

A benchmark \(\beta\) is established to mark any device as a potential source of poisoning if the accuracy difference \(e\) exceeds \(\beta\). 
\subsection{Model Variance and Detection Effectiveness}
\label{sec:analysis}

This section offers formal rationale for the high efficacy of our proposed anomaly detection approach in identifying FL poisoning attacks in edge networks, regardless of the proportion of malicious actors and the presence of coordinated poisoning attacks.

We first establish that the variance in parameters transmitted from each terminal during FL training is lower for common layers compared to private layers. We then validate the high efficiency of the common layer parameter detection method against poisoning attacks under a semi-honest model.

\begin{lemma}
\label{lemma1}
Assuming data distributions are similar yet distinct, if training starts with \(d_{c}^{init} = d_{ce}^{init} = d_{p}^{init}\), it will lead to \(d_{c} \leq d_{ce} \leq d_{p}\) during training. Here, \(d_{c}\) and \(d_{p}\) refer to the distances from the benchmark for common and private layer weights, respectively, and \(d_{ce}\) is the outcome with the complete layer weights of the classical FL model.
\end{lemma}

\begin{proof}
Given the non-IID nature of the training data, the diversity will be reflected in \(w_{p}\), causing an increase in \(w_{p}\) variance without federal aggregation. Classical federal aggregation can transfer this diversity to the parameters, reducing \(d_{ce}\) over time, hence \(d_{ce} \leq d_{p}\). Moreover, \(d_{c}\) reflects data similarity on a smaller scale, resulting in a smaller divergence, thus \(d_{c} \leq d_{ce}\).
\end{proof}

\begin{theorem}
Consider a malicious terminal creating harmful model parameters \(w_{m} + \delta_{m}\), where \(\delta_{m}\) is the deviation from normal to poisoning weights. The CLMD method will more effectively detect the poisoning device, i.e., \(\frac{|\delta_{m}^{c} - \delta_{avg}^{c'}|}{\delta_{avg}^{c'}} > \frac{|\delta_{m}^{ce} - \delta_{avg}^{ce}|}{\delta_{avg}^{ce'}} > \frac{|\delta_{m}^{p} - \delta_{avg}^{p'}|}{\delta_{avg}^{p'}}\), where \(\delta_{m}^{c}\), \(\delta_{m}^{p}\), and \(\delta_{m}^{ce}\) represent the deviation of the malicious device from the benchmark in terms of common layer weights, privacy layer weights, and traditional FL model weights, respectively.
\end{theorem}

\begin{proof}
As the malicious terminal does not actively participate in training, its parameter deviation from the benchmark remains constant, thus \(\delta_{m}^{c} = \delta_{m}^{ce} = \delta_{m}^{p}\). Assuming most devices are normal, Lemma \ref{lemma1} indicates \(\delta_{avg}^{c} \leq \delta_{avg}^{ce} \leq \delta_{avg}^{p}\), leading to \(\frac{|\delta_{m}^{c} - \delta_{avg}^{c}|}{\delta_{avg}^{c}} > \frac{|\delta_{m}^{ce} - \delta_{avg}^{ce}|}{\delta_{avg}^{ce}} > \frac{|\delta_{m}^{p} - \delta_{avg}^{p}|}{\delta_{avg}^{p}}\).
\end{proof}

In conclusion, our proposed method is demonstrated to be exceptionally effective in detecting anomalies.

\section{Experimental Validation}

To assess the efficacy of our proposed negotiation and debate mechanism in small AI models and the secure parameter transfer method for large-scale models, we have structured our experiments as follows:

\subsection{Experimental Configuration}

Our experimental environment comprised a server cluster equipped with state-of-the-art GPUs. TensorFlow was employed as the foundational deep learning framework. Datasets used were diverse, including ImageNet for image classification, along with OpenSubtitles and CommonCrawl for text analysis tasks.

The study was conducted in two distinct phases:

\begin{itemize}
    \item \textbf{Phase 1 - Small AI Model Interaction:} Utilizing the Transformer architecture, a baseline model was developed and then segmented into several smaller models. ImageNet, OpenSubtitles, and CommonCrawl were used for various classification tasks, focusing on small AI model interactions.
    \item \textbf{Phase 2 - Secure Parameter Transfer in Large Models:} This phase integrated a Bert-like model to test secure parameter transfer, with a focus on privacy computing. Evaluation metrics included transfer speed, encryption and decryption latency, and overall security of data transmission.
\end{itemize}

Our primary metrics for evaluation included model accuracy, rate of convergence, and the efficiency of computational and communicational processes, providing a comprehensive understanding of both the small AI model interaction and the secure data transfer mechanism in larger models.

\subsection{Performance Analysis of Collaborative Computing with Small Models}

In this research, we delved into the effects of various model configurations and parameter transmission strategies on model accuracy and privacy. Our findings, illustrated in Fig. \ref{F_1}, reveal how accuracy evolves across different collaborative model setups over a span of 20 training rounds. A notable trend observed is the consistent improvement in accuracy across all configurations, with the rate of improvement accelerating as the degree of collaboration among models increases. Remarkably, the configuration involving four models working collaboratively demonstrated the most significant accuracy enhancement, surpassing the three, two, and single model setups in sequence. In Fig. \ref{F_1}, these configurations are visually differentiated through distinct markers: circles represent a single model, crosses for dual-model collaboration, triangles for three models, and squares for the four-model setup. This visual representation effectively highlights the positive impact of collaborative strategies on boosting model accuracy throughout the various training phases.

\begin{figure}[h]
    \centering
    \includegraphics[width=0.5\textwidth]{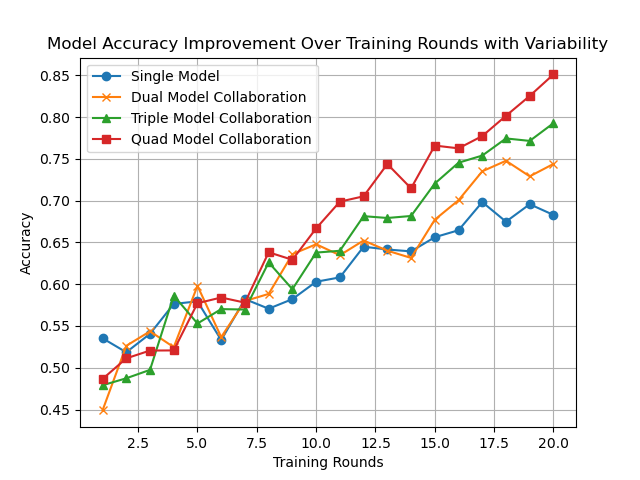}
    \caption{Impact of Model Collaboration on Accuracy}
    \label{F_1}
\end{figure}

In our research illustrated in Fig. \ref{Fi_2}, we explored privacy performance by comparing various model collaboration setups: a single small model, and dual, triple, and quad-model collaborations. The study's results, plotted with the number of training rounds on the x-axis and privacy loss on the y-axis, revealed a significant trend: as the number of collaboratively working models increased, privacy loss consistently decreased, especially in the later stages of training. This finding underscores the efficacy of multi-model collaboration and advanced layered techniques in enhancing privacy safeguards. For clarity in comparison, we used distinct markers in the graph—circles for a single model, crosses for dual models, triangles for triple models, and squares for quad models. This visual representation, as shown in Fig. \ref{Fi_2}, provides robust evidence of the substantial improvement in privacy performance that can be achieved by increasing the number of small models working together and integrating layered technological approaches.
\begin{figure}[h]
    \centering
    \includegraphics[width=0.5\textwidth]{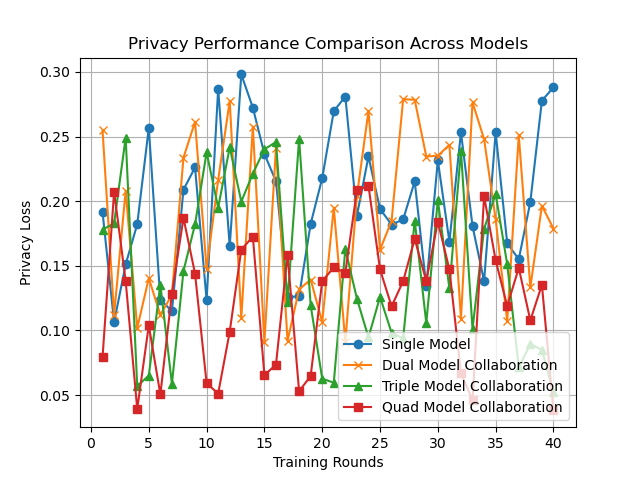}
    \caption{Privacy Preservation in Multi-Model Collaboration}
    \label{Fi_2}
\end{figure}
\subsection{Comparative Analysis of AI Model Configurations}

In our study, depicted in Fig. \ref{Fi_3}, we conducted a comprehensive comparison across various AI model configurations, examining key metrics such as computational requirements, resource consumption, and accuracy. The models are categorized into three types: a standalone large-scale model, individual small-scale models, and a collaborative ensemble of small-scale models. This categorization provides a clear insight into the trade-offs between resource efficiency and model performance.

\subsubsection{Metrics Evaluated}

\textbf{Computational \& Resource Requirements}: This metric measures the computational burden and resource usage of the models. The standalone large-scale model showed significantly higher demands compared to the other configurations, with individually operating small models and the collaborative ensemble being more resource-efficient.

\textbf{Accuracy}: As a vital measure of model performance, we found that the large-scale model has high accuracy, with collaborative small-scale models closely matching its performance. In contrast, individual small models exhibited slightly lower accuracy.

The study presents a key observation: through strategic collaborative approaches among models, it's possible to match or even exceed the performance of a standalone large-scale model while conservatively utilizing resources. This finding offers a promising approach for AI model deployment in settings with limited resources.

\begin{figure}[h]
\centering
\includegraphics[width=0.5\textwidth]{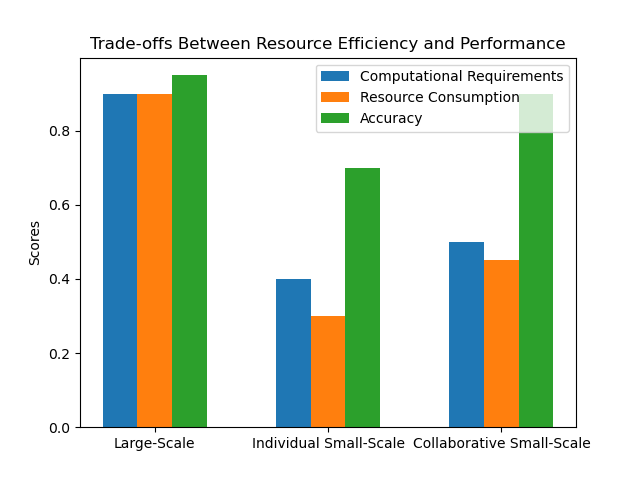}
\caption{Trade-offs Between Resource Efficiency and Performance Across AI Model Configurations}
\label{Fi_3}
\end{figure}

Our findings demonstrate the significant advantages of federated layering and small model collaboration in a 6G network context, achieving enhanced accuracy, improved privacy, and optimal resource utilization.

\subsection{Security analysis of FLT}
In our investigation, we benchmarked our algorithm against renowned methods including Isolation Forest \cite{li2021similarity}, LOF \cite{yang2020anomaly}, FL-MGVN \cite{wu2022fl}, and DÏot \cite{nguyen2019diot}, offering a comprehensive comparative analysis.  

Our analysis in Fig. \ref{ADAC} showcases the proficiency of different algorithms in accurately detecting anomalies within a federated learning network. This comparison is pivotal for ensuring the network's integrity and reliability. The graph illustrates the efficiency of each algorithm in distinguishing true anomalies from false positives over various training rounds. The accuracy trajectory of each algorithm, depicted through their trend lines, offers insights into their effectiveness in anomaly detection. This result underscores our algorithm's enhanced capability in anomaly detection accuracy, a key aspect in ensuring the credibility and operational efficacy of federated learning systems.

In Fig. \ref{ADC}, we focus on the comparative model accuracy achieved by different anomaly detection algorithms. Model accuracy is a critical measure of an algorithm's ability to correctly classify data while reliably identifying anomalies. The graph depicts how model accuracy evolves over training rounds for each algorithm. Superior accuracy indicates that an algorithm effectively balances anomaly detection with the preservation of learning process integrity. This graphical representation is instrumental in evaluating the overarching effectiveness of each algorithm in the federated learning context, highlighting our proposed algorithm's superior performance in maintaining high model accuracy alongside robust anomaly detection.

Fig. \ref{LA} examines the latency of anomaly detection across various federated nodes. Timeliness in anomaly detection, measured by detection latency, is crucial, especially in dynamic environments requiring prompt responses. The graph compares the responsiveness of each algorithm, with lower latency being indicative of quicker anomaly detection. This comparison is essential in assessing the real-time efficiency of our algorithm against existing methods, demonstrating its effectiveness in swiftly addressing anomalous occurrences, thus maintaining the integrity and performance of AI models in urgent federated learning scenarios.

Our study, highlighted in Fig. \ref{MA}, delves into the overall impact of integrating anomaly detection algorithms on the performance of federated learning systems. This comparison is vital for gauging the influence of anomaly detection on the efficiency and accuracy of models within the network. The graph demonstrates how each algorithm, particularly ours, navigates the balance between effective anomaly detection and enhancing model accuracy across training rounds. This analysis aims to showcase our anomaly detection algorithm's dual capability in identifying irregularities and positively influencing the model's learning accuracy, thus bolstering the overall system performance in federated learning applications where precision and effective anomaly detection are paramount.

In conclusion, FLT's superiority in efficiency, resource utilization, and security solidifies its appropriateness for advanced 6G environments, making it a robust solution in the evolving landscape of federated learning.

\begin{figure}[htbp]
	\centering		
	\includegraphics[width=0.5\textwidth]{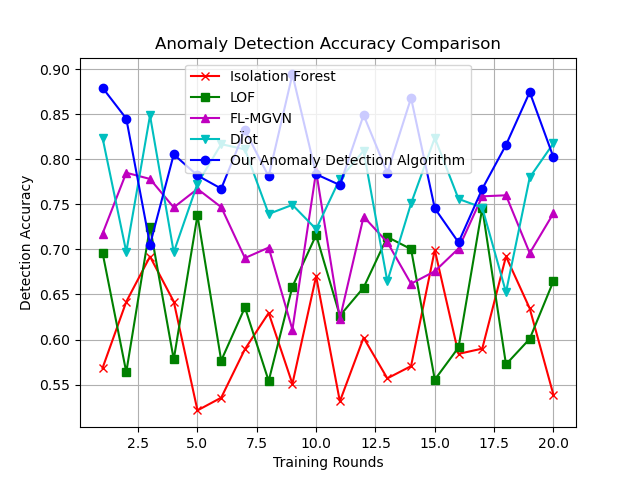}
	\centering
	\caption{Anomaly Detection Accuracy in Federated Learning}\label{ADAC}
\end{figure}
\begin{figure}[htbp]
	\centering		
	\includegraphics[width=0.5\textwidth]{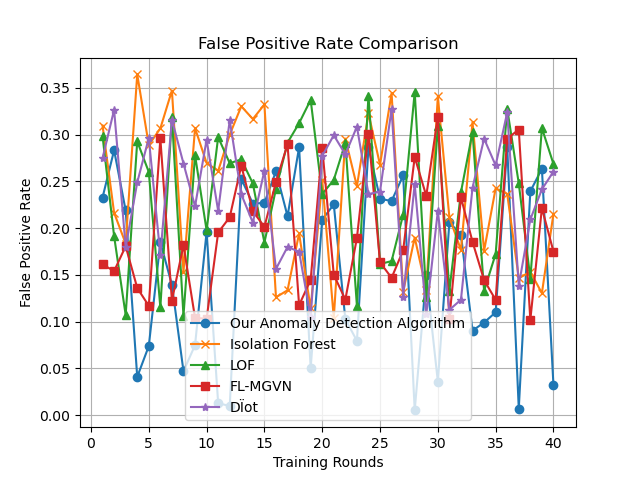}
	\centering
	\caption{Accuracy with Different Anomaly Detection Algorithms}\label{ADC}
\end{figure}
\begin{figure}[htbp]
	\centering		
	\includegraphics[width=0.5\textwidth]{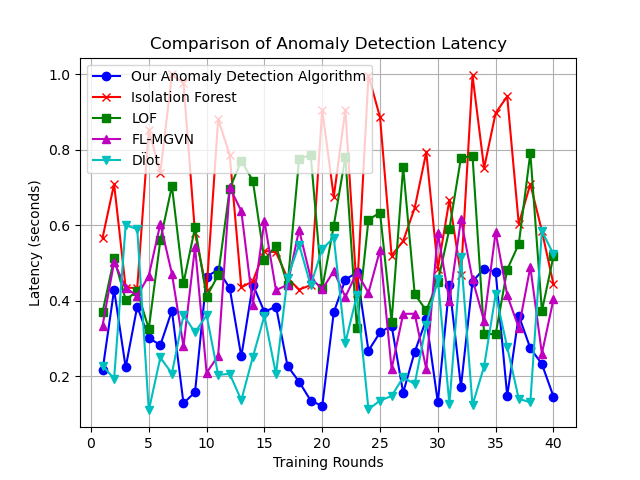}
	\centering
	\caption{Anomaly Detection Latency Across Federated Learning Nodes}\label{LA}
\end{figure}
\begin{figure}[htbp]
	\centering		
	\includegraphics[width=0.5\textwidth]{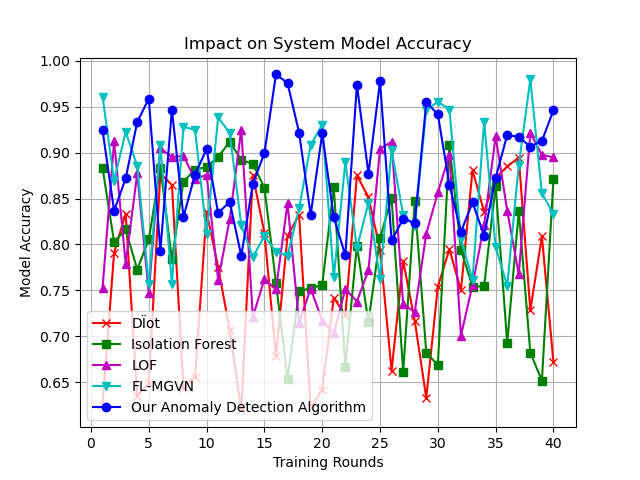}
	\centering
	\caption{Impact of Anomaly Detection on Federated Learning System Performance}\label{MA}
\end{figure}

\subsection{Experimental Validation and Future Perspectives}

The research conducted in this paper is supported by experimental results that demonstrate significant progress in learning efficiency, reasoning accuracy, and privacy protection. The strategies presented not only address current technological challenges but also provide valuable guidance and insights for the future development of information technology. As 6G communication networks continue to mature and proliferate, the proposed technologies are anticipated to find broader applications, fostering technological progress and transformation across various societal domains.

 \section{Conclusions}\label{sec:conclusion}

This paper explores the enhancement of QoS in edge computing by integrating FLT to improve the resilience and efficiency of AI lifelong learning systems. We introduce an innovative approach that enables collaborative interactions among small AI models within a federated layering framework, optimizing their performance in resource-constrained edge environments. This synergy between cloud and edge computing significantly boosts operational efficiency and decision-making efficacy. Our research highlights three key contributions: the utilization of FLT for efficient AI model collaboration, a synergistic cloud-edge architecture that supports resilient lifelong learning systems, and advanced privacy-preserving mechanisms tailored for federated learning in edge settings. Experimental validations confirm the effectiveness of our strategies in enhancing learning efficiency, inference accuracy, and privacy protection, addressing critical challenges in edge computing and AI technologies. This study not only contributes to the academic and practical realms of AI, edge computing, and federated learning but also lays a foundational framework for future advancements in resilient and efficient lifelong learning systems within the evolving landscape of edge computing and 6G networks.

\bibliographystyle{unsrtnat}
\bibliography{my}{}

\end{document}